\documentclass{article}

\usepackage{arxiv}

\usepackage[utf8]{inputenc} 
\usepackage[T1]{fontenc}    
\usepackage{hyperref}       
\usepackage{url}            
\usepackage{booktabs}       
\usepackage{amsfonts}       
\usepackage{nicefrac}       
\usepackage{microtype}      
\usepackage{lipsum}
\usepackage{natbib}
\usepackage{amsmath}
\usepackage{ifpdf}
\usepackage[final]{pdfpages}
\usepackage{accents}
\usepackage{amssymb}
\usepackage{caption}
\usepackage{amsthm}
\usepackage{graphicx}
\usepackage[export]{adjustbox}
\usepackage{epsfig}
\usepackage{dsfont}
\usepackage{tikz-cd}
\usepackage{subfig}
\usepackage{bbm}
\usepackage{wrapfig}
\usepackage{times}

\DeclareMathAlphabet{\mathcal}{OMS}{cmsy}{m}{n}

\DeclareMathOperator{\comp}{comp}

\DeclareMathOperator{\sign}{sign}

\DeclareMathOperator{\Ima}{Im}

\newtheorem{theorem}{Theorem}
\newtheorem{corollary}{Corollary}
\newtheorem{definition}{Definition}
\newtheorem{lemma}{Lemma}

\title{Consistency and Finite Sample Behavior of Binary
Class Probability Estimation}

\author{
   Alexander Mey \\
  Delft University of Technology, The Netherlands\\
  \texttt{a.mey@tudelft.nl} \\
   \And
   Marco Loog\\
  Delft University of Technology, The Netherlands\\
University of Copenhagen, Denmark \\
  \texttt{m.loog@tudelft.nl} \\
}

\begin{document}
\maketitle

\begin{abstract}
In this work we investigate to which extent one can recover class probabilities within the empirical risk minimization (ERM) paradigm. The main aim of this paper is to extend existing results and emphasize the tight relations between empirical risk minimization and class probability estimation. Based on existing literature on excess risk bounds and proper scoring rules, we derive a class probability estimator based on empirical risk minimization. We then derive fairly general conditions under which this estimator will converge, in the $L_1$-norm and in probability, to the true class probabilities. Our main contribution is to present a way to derive finite sample $L_1$-convergence rates of this estimator for different surrogate loss functions. We also study in detail which commonly used loss functions are suitable for this estimation problem and finally discuss the setting of model-misspecification.
\end{abstract}

\section{Introduction}

In binary classification problems we try to predict a label $y \in \{-1,1\}=\mathcal{Y}$ based on an input feature vector $x \in \mathcal{X}$. Since optimizing for the classification accuracy is often computationally too complex, one typically measures performance through a surrogate loss function. Such methods are designed to achieve good classification performance, but often we are also interested in the classifier's confidence or a class probability estimate as such. We may, for instance, not only want to classify a tumor as benign or malignant, but also know an estimated probability that the predicted label is wrong. Also various methods in active or semi-supervised learning rely on such class probability estimates. In active learning they are, for instance, used in uncertainty based rules \cite{AL,AL2} while in semi-supervised learning they can be used for performing entropy regularization \cite{entropy}. 

In this paper we derive necessary and sufficient conditions under which classifiers, obtained through the minimization of an empirical loss function, allow us to estimate the class probability in a consistent way. More precisely, we present a \emph{general} way to derive finite sample bounds based on those conditions. While the use of class probability estimates, as argued before, finds a broad audience, the necessary tools to understand the behavior, especially the literature on proper scoring rules, is not that broadly known. So next to our contribution on finite sample behavior for class probability estimation we present a condensed introduction to this, in our opinion, under-appreciated field.

A proper scoring rule is essentially a loss function that can measure the class probability \emph{point-wise}. We investigate in which circumstances those loss functions make use of this potential and lift this point-wise property to the complete space. Next to proper scoring rules we use \emph{excess risk bounds} to come to our results. Excess risk bounds are essentially inequalities that quantify how much an empirical risk minimizer is off from the true risk. Interestingly, our work does not need any specific excess risk bound and is thus very flexible. Any progress in that theory may also translate to this work. Furthermore, if one is willing to make some restrictions on the underlying distributions, in order to achieve better excess risk bounds, those results directly translate to our setting as well. We elaborate on that further in Section \ref{retrieve2}.

Combining those two areas, our main contributions are the following. Based on the existing literature, we define in Section \ref{estimatordefn}, Equation \eqref{estimator}, a probability estimate $\hat{\eta}$ derived from an empirical risk minimizer. Based on this we analyze in Section \ref{lossfunctions} to which extent commonly used loss functions are suitable for the task of class probability estimation. Following this and the analysis thereafter, we argue in Section \ref{svssh} that the squared loss is, in view of this paper, not a particular good choice.  In Section \ref{main} we derive conditions that ensure that the estimator $\hat{\eta}$ converges in probability towards the true posterior. In the same section we present a general way to analyze the finite sample behavior of the convergence rate for different loss functions, in particular for least squares and logistic regression used for classification, and we discuss the behavior of the estimator when it is misspecified. In this case one can in general not recover the true class probabilities, but instead find the best approximation with respect to a Bregman divergence. As a direct application of our theory we derive in Section \ref{rate} an estimation rate for the squared and the logistic loss, where we believe in particular that the rate for the logistic loss is not known in literature. In Section \ref{discussion} we conclude and discuss our analysis. In particular we discuss how one can extend this work to asymmetric loss functions and analyze their convergence behavior per class label. The following two sections start with related work and some preliminaries. 

\section{Related Work} \label{related}
Many results on posterior estimation in the context of non-parametric regression can be found in \citet{posteriorbook}. The main differences from our results to those type of results is threefold. The first difference is conceptual. While the results presented in \citet{posteriorbook} investigate methods that are specifically designed for posterior estimation, we ask the question if it is possible to obtain consistent posterior estimates with classification methods. Second, to obtain meaningful convergence rate guarantees, the results of \citet{posteriorbook} make usually assumptions on the distribution. We shift this burden from the distribution to the hypothesis set used. The difference is, that while we always have meaningful finite sample guarantees, our estimation procedure is not consistent in the case of model misspecification. The methods used by \citet{posteriorbook} are always consistent, but may have arbitrarily slow convergence on some distributions. Third, as we assume that the excess risk bounds we use are true with high probability over drawn samples, our convergence results hold with high probability, while \citet{posteriorbook} makes those statements in expectation over the sampling process.

The starting point of our analysis follows closely the notation and concepts as described by \citet{buja} and \citet{reid1, reid2}. While \citet{buja} and \citet{reid1} deal with the inherent structure of proper scoring rules, \citet{reid2} make connections between the expected loss in prediction problems and divergence measures of two distributions. In contrast to that we investigate under which circumstances proper scoring rules can make use of their full potential in order to estimate class probabilities. Similar to our work \cite{surrogateregretbounds} gather different sources, in addition to the theory of proper scoring rules, and present general results on regret bounds for class probability estimation. Our work strongly differs in techniques used and thus also in the type of result. \cite{surrogateregretbounds} use an integral representation of the Bayes risk and derive point-wise regret bounds on the Bregman divergence (as in Theorem 3). We draw from the literature of learning theory and excess risk bounds and derive high-probability $L_1$ regret bounds.

\citet{convergence} perform an analysis similar to ours as they also investigate convergence properties of a class probability estimator, their start and end point are very different though. While we start with theory from proper scoring rules, their paper directly starts with the class probability estimator as found in \citet{zhang}. The problem is that the estimator in \citet{zhang} only appears as a side remark, and it is unclear to which extent this is the best, only or even the correct choice. This paper contributes to close this gap and answers those questions. 
They show that the estimator converges to a unique class probability model. In relation to this one can view this paper as an investigation of this unique class probability model and we give necessary and sufficient conditions that lead to convergence to the true class probabilities. Note also that their paper uses convex methods, while our work in comparison draws from the theory of proper scoring rules.

\citet{elicitation} look at the problem in a more general fashion. They connect different surrogate loss functions to certain statistics of the class probability distribution, e.g. the mean, while we focus on the estimation of the full
class probability distribution. This allows us to come to more specific results, such as finite sample behavior.

Another general analysis can be found in \cite{steinwart}. He presents a general tool to relate convergence in a surrogate risk to the convergence in a target risk, and also presents finite sample rates. As we focus on class probability estimation we are able to derive more specific results, and in particular our Lemma \ref{Probineq} and Corollary \ref{modulus} tell us when condition (12) of Theorem 2.13 from \cite{steinwart} is true for class probability estimation.

The probability estimator we use also appears in \citet{Agarwal2} where it is used to derive excess risk bounds, referred to as surrogate risk bounds, for bipartite ranking. The methods used are very similar in the sense that these are also based on proper scoring rules. The difference is again the focus, the conditions used and the conclusions made. They introduce the notion of strongly proper scoring rules which directly allows one to bound the $L_2$-norm, and thus the $L_1$-norm, of the estimator in terms of the excess risk. We show that convergence can be achieved already under milder conditions. We then use the concept of modulus of continuity, of which strongly proper scoring rules are a particular case, to analyze the rate of convergence for posterior estimation. \citet{Agarwal2} on the other hand derives risk bounds for the ranking error, which essentially measures the probability that a randomly drawn positive instance gets assigned a lower value (called score in that context) than a randomly drawn negative instance.

\section{Preliminaries}\label{setup}
We work in the classical statistical learning setup for binary classification. We assume that we observe a finite i.i.d. sample $(x_i,y_i)_{1 \leq i \leq n}$ drawn from a distribution $P$ on $\mathcal{X} \times \mathcal{Y}$. Here $\mathcal{X}$ denotes a feature space and $\mathcal{Y}=\{-1,1\}$ denotes a binary response variable. We then decide upon a hypothesis class $\mathcal{F}$ such that every $f \in \mathcal{F}$ is a map $f: \mathcal{X}  \to \mathcal{V}$ for some space $\mathcal{V}$. Given the space $\mathcal{V}$ we call any function $l:\{-1,1\} \times \mathcal{V} \to [0,\infty)$ a \emph{loss function}. The interpretation of the loss function is that we incur the penalty $l(y,v)$ when we predicted a value $v$ while we actually observed the label $y$. Our goal is then to find a predictor $f_n \in \mathcal{F}$ based on the finite sample such that $\mathbb{E}[l(Y,f_n(X)]$ is small, where $X \times Y$ is a random variable distributed according to $P$. In other words, we want to find an estimator $f_n$ that approximates the true risk minimizer $f_0$ well in terms of the expected loss, where 

\begin{equation} \label{truerisk}
 f_0: = \arg \min\limits_{f \in \mathcal{F}} \mathbb{E}[l(Y,f(X))].
\end{equation}

The estimator $f_n$ is often chosen to be the empirical risk minimizer
$$
f_n=\arg \min\limits_{f \in \mathcal{F}} \sum_{i=1}^n l(y_i,f(x_i)).
$$
 As we show in this paper, finding such an $f_n$ implicitly means to find a good estimate for $p(y\mid x):=P(Y=y\mid X=x)$ in many settings. Since we regularly deal with $p(y\mid x)$ and related quantities we introduce the following notation. To start with, we define $\eta(x):=P(Y=1\mid X=x)$. Depending on the context we drop the feature $x$ and think of $\eta \in [0,1]$ as a scalar. Accepting the small risk of overloading the notation we sometimes also think of $\eta$ as a Bernoulli distribution with outcomes in $\mathcal{Y}$ and parameter $\eta$, as in the following definition.
We define the \emph{point-wise conditional risk} as
\begin{equation} \label{condrisk}
L(\eta,v):=\mathbb{E}_{Y \sim \eta}[l(Y,v)]= \eta l(1,v)+(1-\eta) l(-1,v),
\end{equation}
the \emph{optimal point-wise conditional risk} as
\begin{equation} \label{optcondrisk}
L^*(\eta):= \min\limits_{v \in \mathcal{V}} L(\eta,v),
\end{equation}
and we denote by $v^*(\eta)$ the set of values that optimize the point-wise conditional risk
\begin{equation} \label{fstar}
v^*(\eta):= \arg \min\limits_{v \in \mathcal{V}} L(\eta,v).
\end{equation}
Finally we define the \emph{conditional excess risk} as
\begin{equation} \label{condexcessrisk}
\Delta L(\eta,v):= L(\eta,v)-L^*(\eta).
\end{equation}

\subsection{Proper Scoring Rules}

If we chose $\mathcal{V}=[0,1]$, we say that $l:\{-1,1\} \times \mathcal{V} \to \mathbb{R}$ is a \emph{CPE loss}, where CPE stands for class probability estimation. The name stems from the fact that if $\mathcal{V}=[0,1]$ it is already normalized to a value that can be interpreted as a probability. If $l$ is a CPE loss we call it a \emph{proper scoring rule} or \emph{proper loss} if $\eta \in v^*(\eta)$ and we call it a \emph{strictly proper scoring rule} or \emph{strictly proper loss} if $v^*(\eta)=\{ \eta \}$. In other words, $l$ is a proper scoring rule if $\eta$ is \emph{a} minimizer of $L(\eta,\cdot)$ and this is strict if $\eta$ is the only minimizer. In case $l$ is strict we drop the set notation of $v^*$, so that $v^*(\eta)=\eta$.

\subsection{Link Functions}
As we will see later strictly proper CPE losses are well suited for class probability estimation. In general, however, we cannot expect that $\mathcal{V}=[0,1]$, but we may still want to use the corresponding loss function for class probability estimation. To do that we will use the concept of link functions \cite{buja,reid1}. A \emph{link function} is a map $\psi:[0,1] \to \mathcal{V}$, so a function that indeed links the values from $\mathcal{V}$ to something that can be interpreted as a probability. Combining such a link function with a loss $l:\{-1,1\} \times \mathcal{V} \to [0,\infty)$ one can define a CPE loss $l_{\psi}$ as follows.
\begin{align*} 
& l_{\psi}:\{-1,1\} \times [0,1] \to [0,\infty) \\
& l_{\psi}(y,q):=l(y,\psi(q))
\end{align*}
We call the combination of a loss and a link function $(l,\psi)$ a \emph{(strictly) proper composite loss} if $l_{\psi}$ is (strictly) proper as a CPE loss.

To distinguish between the losses $l$ and $l_{\psi}$ we subscript the quantities \eqref{condrisk}-\eqref{condexcessrisk} with a $\psi$ if we talk about $l_{\psi}$ instead of $l$. For example we define $L_{\psi}(\eta,q):=L(\eta,\psi(q))$ for $q \in [0,1]$ and in the same way we define $v^*_{\psi}(\eta)$, $L^*_{\psi}(\eta)$  and $\Delta L_{\psi}(\eta,q)$. Note that if $(l,\psi)$ is a strictly proper composite loss, we know that $v^*_{\psi}(\eta)$ are single element sets, but the same does not need to hold for $v^*(\eta)$.

\subsection{Degenerate Link Functions}
To ask a composite loss $(l,\psi)$ to be proper is not a strong requirement, one can check that choosing $\psi$ as constant function already fulfills this. This is because a composite loss $(l,\psi)$ is proper, iff the true posterior $\eta$ is a minimizer of the conditional risk $L_{\psi}(\eta,\cdot)$, i.e. $\eta \in v^*_{\psi}(\eta)$. If $\psi$ is constant, then so is the conditional risk $L_{\psi}(\eta,\cdot)$ and then every value is a minimizer, so in particular $\eta$ is a minimizer. We want to avoid this degenerate behavior for the task of probability estimation and will ask $\psi$ to cover enough of $\mathcal{V}$ in the following sense. We call a composite loss $(l,\psi)$ \emph{non-degenerate} if for all $\eta \in [0,1]$ we have that $\Ima \psi \cap v^*(\eta) \neq \emptyset$, where $\Ima \psi \subset \mathcal{V} $ is the image of $\psi$ on $[0,1]$. This does not directly exclude constant link functions for example, but consider the following. If $\psi$ is constant and non-degenerate, then there is a single $v=\Ima \psi$ such that $v \in v^*(\eta)$ for all $\eta$. Thus $v$ would always minimize the loss, and we would, irrespectively of the input, always predict $v$. This is of course a property that no reasonable loss function should carry.

\section{Behavior of Proper Composite Losses} \label{estimatordefn}
For our convergence results we will need a loss function to be a strictly proper CPE loss. In this section we investigate how to characterize those loss functions. 

We start by investigating proper CPE loss functions. Our first lemma states that the link functions that turns the loss $l$ into a proper composite loss is already defined by the behavior of $v^*$. As this lemma and Lemma \ref{mini2} are straightforward derivations from the definitions, and of no further interest, we refer for the proofs to the supplementary material.

\begin{lemma} \label{mini1}
Let $l:\{-1,1\} \times \mathcal{V} \to [0,\infty)$ be a loss function and $\psi$ be a link function. The composite loss function $(l,\psi)$ is then proper and non-degenerate if and only if $\psi \in v^*$, meaning that $\psi(\eta) \in v^*(\eta)$ for all $\eta \in [0,1]$. 
\end{lemma}
\begin{proof}

First we show that if $(l,\psi)$ is proper and non-degenerate, then $\psi \in v^*$.
Let $(l,\psi)$ be a proper composite loss, so $\eta \in v^*_{\psi}(\eta)$, i.e. $\eta$ minimizes $L(\eta,\psi(\cdot))$. As $(l,\psi)$ is non-degenerate there exists at least one $\eta_1$ such that $\psi(\eta_1) \in v^*(\eta)$. If $\psi(\eta) \not\in v^*(\eta)$ we would find that $\eta$ can not be a minimizer of $L(\eta,\psi(\cdot))$ as then $L(\eta,\psi(\eta_1))<L(\eta,\psi(\eta))$. \\

Now we show that $(l,\psi)$ is a proper non-degenerate composite loss if $\psi \in v^*$. By definition, $(l,{\psi})$ is proper if $\eta \in v^*_{\psi}(\eta)$. This is the case if and only if $L(\eta,\psi(\eta))= \min\limits_{q \in [0,1]} L(\eta,\psi(q)) $. But this is the case if $\psi \in v^*$ since $v^*(\eta)$ is defined as the set of minimizers of $L(\eta,\cdot)$. The non-degenerate follows directly by definition. 
\end{proof}
This lemma gives thus necessary and sufficient condition on our link $\psi$ to lead to a proper loss function. The result is very similar to Corollary 12 and 14 found in \cite{reid1}. Their corollaries state necessary and sufficient conditions on the link function, using the assumption that the loss has differentiable partial losses, which is an assumption we don't require. 
 
In Section \ref{retrieve2} we show that \emph{strictly} proper losses, together with some additional assumptions, lead to consistent class probability estimates. So it is useful to know how to characterize those functions. The following lemma shows that a link function that turns a loss into strictly proper and non-degenerate CPE loss can be characterized again by the behavior of $v^*$.

 \begin{lemma} \label{mini2}
Let $l:\{-1,1\} \times \mathcal{V} \to [0,\infty)$ be a loss function and $\psi$ a link function. A composite loss function $(l,\psi)$ is then strictly proper and non-degenerate if and only if $\psi \in v^*$ and $v^*({\eta_1}) \cap v^*({\eta_2}) \cap \Ima \psi=\emptyset$ for all pairwise different $\eta_1,\eta_2 \in [0,1]$.
 \end{lemma} 
 \begin{proof}
 
 By definition the composite loss is strictly proper if and only if 
 $
 \eta=v^*_{\psi}(\eta). 
 $ 
 First we show that $(l,\psi)$ is strictly proper and non-degenerate if $\psi \in v^*$ and $v^*({\eta_1}) \cap v^*({\eta_2})=\emptyset$ for all $\eta_1,\eta_2 \in [0,1]$. From Lemma \ref{mini1} we know already that $\eta \in v^*_{\psi}(\eta)$, we only have to show that $\eta$ is the only element in the set. For that assume that it is not the only element, so that there is a $\gamma \in [0,1]$ such that $\gamma \in v^*_{\psi}(\eta)$. As in the proof of Lemma \ref{mini1} one can conclude that $\psi(\gamma) \in v^*(\eta)$. But we also know, again from Lemma \ref{mini1}, that $\psi(\gamma) \in v^*(\gamma)$. That means that $\psi(\gamma) \in v^*(\eta) \cap v^*(\gamma) \cap \Ima \psi \neq \emptyset$, which is a contradiction to our assumption. \\

 Now we show that $\psi \in v^*$ and $v^*({\eta_1}) \cap v^*({\eta_2}) \cap \Ima \psi=\emptyset$ for all $\eta_1,\eta_2 \in [0,1]$ if $(l,\psi)$ is strictly proper and non-degenerate. The relation $\psi \in v^*$ follows again from Lemma \ref{mini1}. We prove the second claim by contradiction and assume that there exist $\eta_1, \eta_2, \eta_3 \in [0,1]$, all pairwise different, such that $\psi(\eta_3) \in v^*({\eta_1}) \cap v^*({\eta_2})$. With this choice and using that $\psi$ is strictly proper it follows that $\eta_3 = v^*_{\psi}(\eta_1)$ and $\eta_3 = v^*_{\psi}(\eta_2)$. That means that $\eta_1=\eta_3=\eta_2$ which is a contradiction. 
 \end{proof}
So if $(l,\psi)$ is a strictly proper composite loss it will fulfill some sort of injectivity condition on the sets $v^*(\eta)$. With this we will be able to define an inverse $\psi^{-1}$ on those sets, and this will be essentially our class probability estimator. With Lemma \ref{mini2} we can connect every $v \in \mathcal{V}$ to a unique $\eta_v$ by the unique relation $v \in v^*(\eta_v)$ if we assume that $v^*$ \emph{disjointly covers} $\mathcal{V}$ in the sense that 
  \begin{align} \label{cover}
 & \bigcup\limits_{\eta \in [0,1]} v^*(\eta)=\mathcal{V} \ \ \ \text{and} \\ 
 & v^*(\eta_1) \cap v^*(\eta_2)=\emptyset \ \ \ \forall \ \eta_1, \eta_2 \in [0,1], \ \ \eta_1 \neq \eta_2. \label{disjoint}
  \end{align}
Note that we know from Lemma \ref{mini2} that for strict properness it is sufficient for $(l,\psi)$ that the disjoint property \eqref{disjoint} only holds on $\Ima \psi$, the image of $\psi$. This is merely a technicality and we will assume from now on that every strictly proper composite loss will satisfy \eqref{disjoint}. The covering property \eqref{cover} on the other hand can be violated. This happens for example if we use the squared loss together with $\mathcal{V}=\mathbb{R}$. For the squared loss $v^*(\eta)=2\eta-1$, so it only covers the space $[-1,1]$.

If we assume, however, that the regularity properties \eqref{cover} and \eqref{disjoint} hold for a strictly proper non-degenerate composite loss $(l,\psi)$ we can extend the domain of $\psi^{-1}$ from $\Ima \psi$ to the whole of $\mathcal{V}$, see also Figure \ref{psimap}. 
 \begin{figure} 
 \begin{center}
  \includegraphics[width=0.7\textwidth]{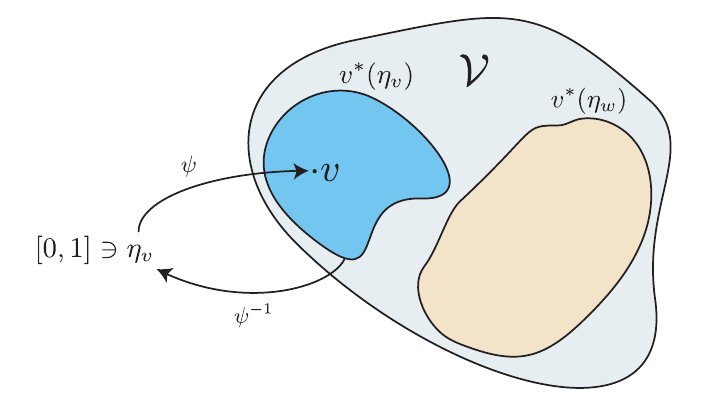}
   \end{center}
  \caption{The way we generally think of the mapping $\psi$, $\psi^{-1}$ and the sets $v^{*}$ if $(l,\psi)$ is non-degenerate and strictly proper. In those cases we can extend $\psi^{-1}$ to the sets $v^*$. This is well defined as the sets $v^*(\eta_v)$ and $v^*(\eta_w)$ have empty intersection for different $\eta_v, \eta_w \in [0,1]$. Note that Lemma \ref{mini2} guarantees that $\psi(\eta_v) \in v^*(\eta_v)$.}
  \label{psimap}
 \end{figure}

\begin{definition} \label{extendinv}
Let $(l,\psi)$ be a strictly proper, non-degenerate composite loss and assume that $v^*$ disjointly covers $\mathcal{V}$. We define, by abuse of notation, the inverse link function $\psi^{-1}: \mathcal{V} \to [0,1]$ by
$\psi^{-1}(v)=\eta_v$, where $\eta_v$ is the unique element in $[0,1]$ such that $v \in v^*(\eta_v)$.
\end{definition}

The requirements from the previous definition is what we consider the archetype of a composite loss that is suitable for probability estimation, although not all of the requirements are necessary. This motivates the following definition.

\begin{definition}

We call a composite loss $(l,\psi)$ a natural CPE loss if $\psi$ is non-degenerate, $v^*$ fulfills the disjoint cover property \eqref{cover} and \eqref{disjoint} and $(l,\psi)$ is strictly proper. 
\end{definition}
We now have all the necessary work done to make the following observation.
\begin{corollary} \label{cpeloss}
If $(l,\psi)$ is a natural CPE loss, then $\psi^{-1}={v^*}^{-1}$.
\end{corollary}

 The corollary tells us that we can optimize our loss function over $\mathcal{V}$ to get $v^*(\eta)$ and then map this back with the inverse link $\psi^{-1}$ to restore the class probability $\eta$. For this we once more refer to Figure \ref{psimap}. Remember that the set $v^*(\eta_v)$ is the set of all $v \in \mathcal{V}$ that minimize the loss if the true posterior probability was $\eta_v$. If we use a natural CPE loss $(l,\psi)$ we know then that $\psi^{-1}$ maps all those points back to $\eta_v$.

 Given a predictor $f: \mathcal{X} \to \mathcal{V}$ this motivates to define an estimator of $\eta(x)$ as 
\begin{equation} \label{estimator}
 \hat{\eta}=\hat{\eta}(x)=\psi^{-1}(f(x)).
\end{equation}

In Section \ref{main} we give conditions under which $\hat{\eta}(x)$ converges in probability towards $\eta(x)$ when using an empirical risk minimizer $f_n$ as a prediction rule. More formally; Given any $\epsilon>0$ we show that under certain conditions $\hat{\eta}_n(x):=\psi^{-1}(f_n(x))$ satisfies
 \begin{equation} \label{convergence}
 P(|\hat{\eta}_n(X)-\eta(X)|> \epsilon) \xrightarrow{n \to \infty} 0,
\end{equation}
where the probability is measured with respect to $P$. In the next section, however, we want to investigate first $v^*$ and ${v^*}^{-1}$ for some commonly used loss functions.
\section{Analysis of Loss Functions} \label{lossfunctions}
Throughout this paper we consider the following loss functions: Squared loss (Sq), logistic loss (log), squared hinge loss (SqH), Hinge loss and the $0$-$1$ loss, see the first two columns of Table \ref{lossfunc} for specifications. Table \ref{linkfunc} shows the link function that turns the loss functions into a strictly proper composite loss, if possible. Note that this can be decided with the help of Lemma \ref{mini2} and the functions $v^*(\eta)$ which are also shown in Table \ref{linkfunc}. We note that the behavior of the squared and squared hinge loss seems to be very similar. In Section \ref{svssh}, however, we point out an important difference.
\renewcommand{\arraystretch}{1.5}
\begin{table} [ht!]
\centering
\caption{ The loss functions we consider in this paper. The function ${v^*}^{-1}(v)$ is the function that transforms a real output to a posterior estimate. }
\vspace{0.1in}
\label{lossfunc}
\begin{tabular}{lll}
\toprule
    Loss   & $l(v,y)$  & ${v^*}^{-1}(v)$    \\
    \midrule
Sq     & $(1-yv)^2$ & $\frac{v+1}{2}$   \\
Log     &$\ln (1+e^{-vy})$ & $\frac{1}{1+e^{-v}}$  \\
& \\
SqH & $\max (0,1-vy)^2$ & $T(\frac{v+1}{2})$    \\
& \\
Hinge      & $\max (0,1-vy)$ & $ 
\begin{cases}
 \frac{1}{2} & v \in (-1,1) \\
         (0,\frac{1}{2}) & v=-1 \\
   (\frac{1}{2},1) & v=1\\
    1, & v > 1 \\
    0,      & v <  -1 
  \end{cases}$   \\
& \\
0-1       & $I_{ \{\sign(vy) \neq 1\}}$ & $
\begin{cases}
   [\frac{1}{2},1] & \text{if } v \in (0,\infty) \\
     [0, \frac{1}{2}], &  v \in (-\infty,0) \\
      \frac{1}{2}, &  v = 0 \\
\end{cases}$   \\
\bottomrule
\end{tabular}
\end{table}
As already noted by \cite{buja}, also Table \ref{linkfunc} shows that the hinge loss is not suitable for class probability estimation. We observe that the intersections of $v^*(\eta)$ for different $\eta \in [0,1]$ are not disjoint. By Lemma \ref{mini2} we can conclude that there is no link $\psi$ such that $(l,\psi)$ is strictly proper. One way to fix this, proposed by \cite{duin1998classifier} and similar by \cite{Platt}, is to fit a logistic regressor on top of the support vector machine. \cite{bartlett2} investigate the behavior of the hinge loss deeper by connecting the class probability estimation task to the sparseness of the predictor. The hinge loss is of course classification calibrated (essentially meaning that we find point-wise the correct label with it), so between our considered surrogate losses it is the only one that really directly solves the classification problem without implicitly estimating the class probability. 
\begin{table} [h!]
  \caption{The different loss functions we consider in this paper together with their link functions that turn them into CPE losses (if possible).}
  \vspace{0.1in}
  \label{linkfunc}
  \centering
  \begin{tabular}{llll}
    \toprule
    Loss      & $\psi(\eta)$  & $v^*(\eta)$   \\
    \midrule
Sq   &  $2 \eta -1$&$2 \eta -1$   \\ 
Log     &$\ln \frac{\eta}{1-\eta}$& $\ln \frac{\eta}{1-\eta}$  \\ 
& \\
SqH  & $2 \eta -1$& $ 
\begin{cases}
   2\eta-1, & \eta \in (0,1) \\
    [1,\infty), & \eta=1  \\
    (-\infty,-1],      & \eta=-1
\end{cases}
$ \\ 
& \\
Hinge          &  - &$ 
\begin{cases}
   \sign (2\eta-1), & \eta \in (0,1)\setminus \frac{1}{2} \\ 
   [-1,1]  & \eta=\frac{1}{2} \\
    [1,\infty), & \eta=1  \\
    (-\infty,-1],      & \eta=-1
\end{cases}$ \\  
  & \\
0-1 & -&$\begin{cases} 
(0,\infty), & \eta \in (\frac{1}{2},1] \\
(-\infty,0), & \eta \in [0,\frac{1}{2}) \\
\mathbb{R}, & \eta = \frac{1}{2} 
\end{cases} $ \\
    \bottomrule
  \end{tabular}
\end{table}
\section{Convergence of the Estimator} \label{main}
We now prove that the estimator $\hat{\eta}(x)$ as defined in Equation \eqref{estimator} converges in probability and in the $L_1$-norm to the true class probability $\eta$ whenever we use an empirical risk minimizer, for which we have excess risk bounds, as a prediction rule.

\subsection{Using the True Risk Minimizer for Estimation} \label{retrieve}
Before we can investigate under which conditions an empirical risk minimizer can (asymptotically) retrieve $\eta(x)$ we need to investigate under which conditions the true risk minimizer can retrieve it. In this subsection we formulate a theorem that gives necessary and sufficient conditions for that. Not surprisingly we basically require that our hypothesis class is rich enough so as to contain the class probability distribution already. \cite{bartlett1} and similar works often avoid problems caused by restricted classes by assuming from the beginning that the hypothesis class consists of all measurable functions. This theorem relaxes this assumption for the purpose of class probability estimation. 

In this setting we assume that we use a hypothesis class $\mathcal{F}$ where $f \in \mathcal{F}$ are functions $f: \mathcal{X} \to \mathcal{V}$. If we want to do class probability estimation we rescale those functions by composing them with the inverse link $\psi^{-1}: \mathcal{V} \to [0,1]$ so that we effectively use the hypothesis class $\psi^{-1}(\mathcal{F}):=\{ \psi^{-1} \circ f \mid  f \in \mathcal{F} \}$. We then get the following theorem about the possibility of retrieving the posterior with risk minimization.

\begin{theorem}  \label{First}
Assume that $(l,\psi)$ is a natural CPE loss function.  Let $$f_0 = \arg \min\limits_{f \in \mathcal{F}} \mathbb{E} [l(Y,f(X)].$$ 
Then $\psi^{-1}(f_0(x))=\eta(x)$ almost surely if and only if
$\eta \in \psi^{-1}(\mathcal{F}).
$
\end{theorem}

 Following Theorem \ref{First} we need to assume that our hypothesis class is flexible enough for consistent class probability estimation. We formulate this assumption as follows.
\paragraph{Assumption A} Given a natural CPE loss $(l,\psi)$ we assume that $\eta \in \psi^{-1}(\mathcal{F})=\{ \psi^{-1} \circ f \mid  f \in \mathcal{F} \}.$
In Subsection \ref{Misspec} we will deal with the case of misspecification, i.e. when $\eta \notin \psi^{-1}(\mathcal{F})$.

\subsection{Using the Empirical Risk Minimizer for Estimation} \label{retrieve2}

In the previous section we considered the possibility of retrieving class probability estimates with the true risk minimizer. To move on to empirical risk minimizers we need the notion of excess risk bounds.
\begin{definition}
Let $f_n: \mathcal{X} \to \mathbb{R}$ be any estimator of $f_0 \in \mathcal{F}$, which may depend on a sample of size $n$. We call 
\begin{equation*} 
B^\mathcal{F}(n,\gamma): \mathbb{N} \to [0,\infty)
\end{equation*} an excess risk bound for $f_n$ if for all $\gamma>0$ we have $B^{\mathcal{F}}(n,\gamma) \rightarrow 0$ for $n \rightarrow \infty$ and with probability of at least $1-\gamma$ over the $n$-sample we have
\begin{align*}
& \mathbb{E}_{X} [\Delta L(\eta(X),f_n(X))] \\
=&\mathbb{E}_{X,Y}[l(Y,f_n(X))-l(Y,f_0)] \leq B^{\mathcal{F}}(n,\gamma).
\end{align*}

\end{definition}
Excess risk bounds are typically in the order of $\left(\frac{\comp(\mathcal{F})}{n}\right)^{\beta}$, where $\beta \in [0.5,1]$ and $\comp(\mathcal{F})$ is a notion of model complexity. Common measures for the model complexity are the VC dimension \cite{Vapnik1998}, Rademacher complexity \cite{bartlett3} or $\epsilon$-cover \cite{Benedek}. The existence of excess risk bounds is tied to the finiteness of any of those complexity notions. A lot of efforts in this line of research are made to find relations between the exponent $\beta$ and the statistical learning problem given by $\mathcal{F}$, the loss $l$ and the underlying distribution $P$. Conditions that ensure $\beta>\frac{1}{2}$ are often called easiness conditions, such as the Tsybakov condition \cite{tsybakov} or the Bernstein condition \cite{audibert}. Intuitively those conditions often state that the variance of our estimator gets smaller the closer we are to the optimal solution. For a in-depth discussion and some recent results we refer to the work of \cite{grunwald}.

Excess risk bounds allow us to bound $\Delta L(\eta(x),f_n(x))$ for a loss $l$, so in particular we can bound $\Delta L_{\psi}(\eta(x), \hat{\eta}(x))$ for a composite loss $(l,\psi)$. We will show $L_1$-convergence by connecting the behavior of $\Delta L_{\psi}(\eta(x), \hat{\eta}(x))$ to $|\eta(x)-\hat{\eta}(x)|$. The following lemma introduces a condition that allows us to draw this connection. 

\begin{lemma} \label{Probineq}

Let $(l,\psi)$ be a natural CPE loss. Assume that for all $\eta \in [0,1]$ the maps
$$L^{0}_{\psi}(\eta,\cdot):=L_{\psi}(\eta,\cdot)\restriction_{[0,\eta]}:[0,\eta] \to \mathbb{R}$$ and 
$$L^{1}_{\psi}(\eta,\cdot):=L_{\psi}(\eta,\cdot)\restriction_{[{\eta},1]}:[{\eta},1] \to \mathbb{R}$$ 
 are strictly monotonic, where $L_{\psi}(\eta,\cdot)\restriction_I$ refers to the restriction of the mapping $L_{\psi}(\eta,\cdot)$ to an interval $I$. This is the case iff $L_{\psi}(\eta,\cdot)$ is strictly convex with $\eta$ as its minimizer. Then there exists for all $\epsilon>0$ a $\delta=\delta(\epsilon)>0$ such that for all $\eta,\hat{\eta} \in [0,1]$ 
\begin{equation} \label{ineq}
|\Delta L_{\psi}(\eta,\hat{\eta})| < \delta \Rightarrow  |\eta-\hat{\eta}|< \epsilon .
\end{equation}
\end{lemma}

\begin{proof}

With the assumptions on $L^{0}_{\psi}(\eta,\cdot)$ and $L^{1}_{\psi}(\eta,\cdot)$ we know that ${L^{0}_{\psi}}^{-1}(\eta,\cdot)$ and ${L^{1}_{\psi}}^{-1}(\eta,\cdot)$ exist and are continuous \cite{inverse}. By definition that means that for every $l,\hat{l} \in \Ima L^0_{\psi}(\eta,\cdot)$ and for all $\epsilon>0$ there exists a $\delta>0$ such that
\begin{equation} \label{prooflemma}
|\hat{l}-l |< \delta \Rightarrow  |{L^{0}_{\psi}}^{-1}(\eta,\hat{l}) -{L^{0}_{\psi}}^{-1}(\eta,l)| < \epsilon
\end{equation}
and similar for $L^1_{\psi}(\eta,\cdot)$. W.l.o.g assume now that $\hat{\eta} < \eta$ so that $\hat{\eta} \in [0,\eta]$. Plugging $l=L^0_{\psi}(\eta,\eta)$ and $\hat{l}=L^0_{\psi}(\eta,\hat{\eta})$ into \eqref{prooflemma} we get the following relation.
\begin{align*}
&  |\Delta L_{\psi}(\eta,\hat{\eta})|=|L^0_{\psi}(\eta,\hat{\eta})-L^0_{\psi}(\eta,\eta)| < \delta \\
 \Rightarrow &  |\hat{\eta}-\eta|= |{L^{0}_{\psi}}^{-1}(\eta,\hat{l}) -{L^{0}_{\psi}}^{-1}(\eta,l)| < \epsilon
\end{align*}
\end{proof}

The map $L^{0}_{\psi}(\eta,\cdot)$ captures the behavior of the loss when $\eta$ is the true class probability and we predict a class probability less than $\eta$. Similarly $L^{1}_{\psi}(\eta,\cdot)$ captures the behavior when we predict a class probability bigger than $\eta$, see also Figure \ref{L_loss}. In Corollary \ref{modulus}, further below, we draw a connection between $\delta(\epsilon)$ and the modulus of continuity of the inverse functions of $L^1_{\psi}(\eta,\cdot)$ and $L^0_{\psi}(\eta,\cdot)$. The function $\delta(\epsilon)$ plays an important role in the convergence rate of the estimator $\hat{\eta}(x)$ as described in the next theorem. 

\begin{figure}
\subfloat[The map $L_{\psi}(\eta,\cdot)$ for $\eta=0.2$ and $l$ \newline
being the squared loss.]{\includegraphics[width=0.5\linewidth]{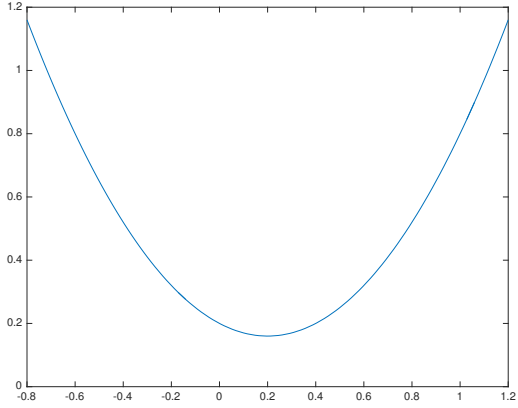}} 
\subfloat[The map $L_{\psi}(\eta,\cdot)$ for $\eta=0.2$ and $l$ being the logistic loss.]{\includegraphics[width=0.5\linewidth]{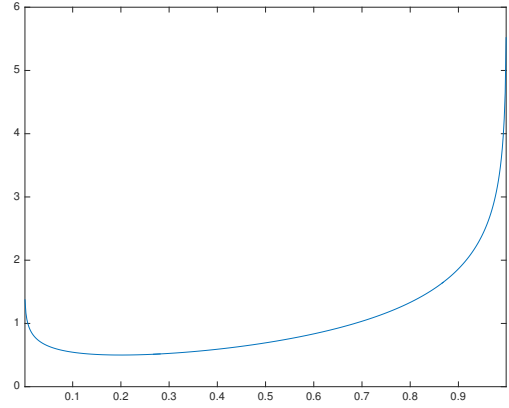}} 
\caption{The map $L_{\psi}(\eta,\cdot)$ for the squared and the logistic loss. The two maps $L^{0}_{\psi}(\eta,\cdot)$ and $L^{1}_{\psi}(\eta,\cdot)$ split it into the parts left and right of $\eta$.}
\label{L_loss}
\end{figure}

\begin{theorem} \label{probbound}
Let $(l,\psi)$ be a natural CPE loss and assume Assumption A holds. Furthermore let $B^{\mathcal{F}}(n,\gamma)$ be an excess risk bound for $f_n$ and assume that $L_{\psi}(\eta,\cdot)$ is strictly convex with $\eta$ as its minimizer. Then there exists a mapping \\ $\delta(\epsilon):[0,1] \to \mathbb{R}$ such that for $\hat{\eta}_n(x):=\psi^{-1}(f_n(x))$ we have with probability of at least $1-\gamma$ that
\begin{equation} \label{rateofconv}
P(|\eta(X)-\hat{\eta}_n(X)| > \epsilon) \leq \frac{B^{\mathcal{F}}(n,\gamma)}{\delta(\epsilon)}.
\end{equation}
\end{theorem}
\begin{proof}
Using Lemma \ref{Probineq} for the first inequality, Markov's Inequality for the second and the excess risk bound for the third inequality it follows that
\begin{align*}
&P(|\eta(X)-\hat{\eta}_n(X) |> \epsilon ) \leq P(\Delta L_{\psi}(\eta(X),\hat{\eta}_n(X)) > \delta) \\
 =& P(\Delta L(\eta(X),f_n(X)) > \delta) \\
\leq & \frac{\mathbb{E}[ \Delta L(\eta(X),f_n(X)) ]}{\delta(\epsilon) }\leq \frac{B^{\mathcal{F}}(n,\gamma)}{\delta(\epsilon)} .
\end{align*} 
\end{proof}
This theorem gives us directly the earlier claimed asymptotic convergence result.
\begin{corollary}
Under the assumptions of Theorem \ref{probbound} we have that $\hat{\eta}_n(x)=\psi^{-1}(f_n(x))$ converges in probability and $L_1$-norm to $\eta(x)$ with probability 1.
\end{corollary}
We do not have to restrict ourselves to asymptotic results though. Theorem \ref{probbound} can also be used to derive rate of convergences as we will see in Subsection \ref{rate}. But before that we briefly want to address the case of misspecification, i.e. the case when Assumption A does not hold.
\subsection{Misspecification} \label{Misspec}
The case of misspecification can be dealt with once we assume that $L^*_{\psi}$ has a gradient. If this holds then \cite{reid1} show the identity
\begin{equation} \label{Bregman}
\Delta L_{\psi}(\eta, \hat{\eta})=D_{-L^*_{\psi}}(\eta, \hat{\eta})
\end{equation}
where $D_{-L^*_{\psi}}(\eta, \hat{\eta})$ is the with $-L^*_{\psi}$ associated Bregman divergence between $\eta$ and $\hat{\eta}$.
Excess risk bounds on $\Delta L_{\psi}(\eta,\hat{\eta})$ translate then into bounds on the Bregman divergence between $\eta$ and $\hat{\eta}$ and asymptotically we approach the best class probability estimate in terms of this divergence. 

\subsection{Rate of Convergence} \label{rate}

For the rate of convergence it is crucial to investigate the function $\delta(\epsilon)$ from Inequality \eqref{rateofconv}. One way to analyze this is to study the modulus of continuity of the inverse functions of $L^{0}_{\psi}(\eta,\cdot)$ and $L^{1}_{\psi}(\eta,\cdot)$:
\begin{definition}

Let $\omega:[0,\infty] \to [0,\infty]$ be a monotonically increasing function. Let $I \subset \mathbb{R}$ be an interval. A function $g: I \to \mathbb{R}$ admits $\omega$ as a modulus of continuity at $x \in I$ if and only if
$$
|g(x)-g(y)| \leq \omega(|x-y|)
$$
for all $y \in I$.
\end{definition}

For example H{\"o}lder and Lipschitz continuity are particular moduli of continuity. This notion allows us a to draw the following connection between $\epsilon$ and $\delta(\epsilon)$.
\begin{corollary} \label{modulus}
Let $(l,\psi)$ be a natural CPE loss and let $\omega:[0, \infty] \to [0,\infty]$ be a monotonically increasing function. Assume that for all $\eta \in [0,1]$ the mappings ${L^0_{\psi}}^{-1}(\eta,\cdot)$ and  ${L^1_{\psi}}^{-1}(\eta,\cdot)$ admit $\omega$ as a modulus of continuity at $\eta$. Then $\delta(\epsilon):=\omega^{-1}(\epsilon)$ is a mapping such that Implication \eqref{ineq} holds.
\end{corollary}
\begin{proof}
W.l.o.g. assume that $\hat{\eta} \in [0,\eta]$. Let $\hat{l}=L^0_\psi(\eta,\hat{\eta})$ and $l=L^0_\psi(\eta,\eta)$. By using that ${L^0_\psi}^{-1}(\eta,\cdot)$ admits $\omega$ as a modulus of continuity we have
$$
|{L^0_\psi}^{-1}(\eta,l)-{L^0_\psi}^{-1}(\eta,\hat{l})| \leq \omega(|l-\hat{l} |).
$$
Plugging in the definition of $\hat{l}$ and $l$ this means that
$$
|\hat{\eta}-\eta| \leq \omega(\Delta L_\psi(\eta,\hat{\eta})).
$$
Using the monotonicity of $\omega$ it follows that if $\Delta L_{\psi}(\eta,\hat{\eta}) \leq \delta(\epsilon)= \omega^{-1}(\epsilon)$, then 
$$|\eta - \hat{\eta}| \leq \omega(\Delta L_\psi(\eta,\hat{\eta})) \leq  \omega(\omega^{-1}(\epsilon))= \epsilon.$$ 
This is exactly the Implication \eqref{ineq}. 
\end{proof}
Note that it follows from the proof that finding a modulus of continuity $\omega$ for ${L^0_\psi}^{-1}(\eta,\cdot)$ and ${L^1_\psi}^{-1}(\eta,\cdot)$ can be done by showing the bound $|\hat{\eta}-\eta| \leq \omega(\Delta L_\psi(\eta,\hat{\eta}))$. We will use that in the following examples, where we analyze $\delta(\epsilon)$ for the squared (hinge) loss and the logistic loss. We show that those loss functions lead to a modulus of continuity given by the square root times a constant. \cite{Agarwal2} calls loss functions that admit this modulus of continuity \emph{strongly-proper} loss functions. The following analysis can thus be found there in more detail and for a few more examples. We will use for simplicity versions of the losses that do not need a link function, and are already CPE losses.

\paragraph{Example: Squared Loss and Squared Hinge Loss} Let $l(y,\hat{\eta})$ be given by the partial loss functions $l(1,\hat{\eta})=(1-\hat{\eta})^2$ and $l(-1,\hat{\eta})=\hat{\eta}^2$. We can derive that $\Delta L(\eta,\hat{\eta})=(\eta-\hat{\eta})^2$. With this we can directly bound $
|\hat{\eta} -\eta| \leq \sqrt{\Delta L(\eta,\hat{\eta})}
$
and thus choose $\delta(\epsilon)$ as the inverse of the square-root function, so that $\delta(\epsilon)=\epsilon^2$. The analysis for the squared hinge loss is the same as this version of the squared loss is already a CPE loss.

\paragraph{Example: Logistic Loss} 
Let $l(y,\hat{\eta})$ be given by the partial loss functions $l(1,\hat{\eta})=-\ln(\hat{\eta})$ and $l(-1,\hat{\eta})=-\ln(1-\hat{\eta})$. One can derive that $\Delta L(\eta,\hat{\eta})=-\eta \ln(\frac{\hat{\eta}}{\eta})-(1-\eta)\ln(\frac{1-\hat{\eta}}{1-\eta})$. In the appendix we show the bound $
|\eta-\hat{\eta}| \leq  \sqrt{\frac{1}{2}\Delta L(\eta,\hat{\eta}})
$, as well as that $\frac{1}{2}$ is the optimal constant, so that we can choose $\delta(\epsilon)=2 \epsilon^2$.

\subsection{Squared Loss vs Squared Hinge Loss} \label{svssh}
In this section we will subscript previously defined entities with $S$ and $SH$ for the squared and square hinge loss respectively. When using squared loss vs the squared hinge loss for class probability estimation there is one big difference in the inverse of the link function, namely its domain. The inverse link function is a map $\psi^{-1}_{S}: \mathcal{V} \to [0,1]$. If we use the square loss we implicitly chose $\mathcal{V}=[-1,1]$ since this is the range of $\psi_S$. The range of $\psi_{SH}$ on the other hand is $\mathcal{V}=\mathbb{R}$. That means that if we want to use the squared loss for class probability estimation we really have to parametrize our prediction functions $f: \mathcal{X} \to [-1,1]$, a simple linear model for example would usually not fit this assumption as the range of those models can be outside of $[-1,1]$. For the squared hinge loss on the other hand we can allow for functions $f: \mathcal{X} \to \mathbb{R}$.

 \cite{clipping} proposes to just truncate the inverse link for the squared loss, so using the same inverse link as for the squared hinge loss. This is fine as long as our hypothesis class is flexible enough, but leads to problems if that is not the case as the following example shows. 
 
 Assume we are given three one-dimensional data points $x_1=-1,x_2=0,x_3=3$ together with their true class probabilities $\eta(x_1)=0, \eta(x_2)=1/3, \eta(x_3)=1$. We want to learn this classification with linear models, which are two-dimensional  after including a bias term. That means that $\mathcal{F}=\{ f: \mathcal{X} \to \mathbb{R} \mid  \exists w_1,w_2 \in \mathbb{R}: f(x)=w_1x +w_2\}$. One can check that in case of the squared hinge loss function we can recover the true class probabilities with the linear function given by $(w_1,w_2)=(2,-\frac{1}{4})$. By Theorem \ref{First} we know then that an optimal solution $f_0$ is also able to recover the true class probabilities. 

The squared loss has after truncating the following problem. Although the linear function $(w_1,w_2)=(2,-\frac{1}{4})$ is part of $\psi^{-1}_{S}(\mathcal{F})$, after truncating, it will not be found back as an optimal solution $f_0$. 
One can instead check that for the given example the true risk minimizer is given by $f_0=(\frac{19}{39},-\frac{17}{39} )$. And this hypothesis does not recover the true class probabilities. This might appear as a contradiction to Theorem \ref{First}. But the problem arises because we use a different link function than the one associated to the square loss.

\section{Discussion and Conclusion} \label{discussion}
The starting point of this paper is the question if one can retrieve consistently a class probability estimate based on ERM in a consistent way. To answer this question we draw from earlier work on proper scoring rules and excess risk bounds. Lemmas \ref{mini1} and \ref{mini2}, our first results, characterize strictly proper composite loss functions in terms of their link function. Based on those lemmas, we subsequently derive fairly general necessary and sufficient conditions for retrieving the true class probability with ERM as formulated in Theorem \ref{First}. We show that to retrieve the true probabilities we essentially need that they are already part of our hypothesis class $\mathcal{F}$.

In Section \ref{main} we show that consistency arises whenever we use strictly proper (composite) loss functions, our hypothesis class is flexible enough, and we have excess risk bounds. This is the case, for example, whenever one of the complexity notions mentioned in Section \ref{main} is finite. We then discuss the relation between the finite sample size behavior of the excess risk bound and the probability estimate and examine this relation for two example loss functions. 

In Lemma \ref{Probineq} we introduce conditions under which a composite loss function $(l,\psi)$ leads to a consistent class probability estimator. In particular we have a condition on the conditional risk $L_{\psi}(\eta,\cdot)$, see also Figure \ref{L_loss}. Based on that we derive in Corollary \ref{modulus} conditions which allow us to analyze the convergence rate for different loss functions. In the corollary we don't distinguish between $L^{0}_{\psi}(\eta,\cdot)$ and $L^{1}_{\psi}(\eta,\cdot)$, which leads to the same convergence rate for predicting values left and right from $\eta$. But the modulus of continuity for those two functions can be really different, especially when using asymmetric proper scoring rules \cite{winkler}. We believe that  by analyzing $L^{0}_{\psi}(\eta,\cdot)$ and $L^{1}_{\psi}(\eta,\cdot)$ individually one can extend our work to analyze the convergence behavior of asymmetric scoring rules in more detail, meaning that one could achieve different rates for over or underestimating a certain posterior level.
\bibliography{posterior}
\bibliographystyle{plainnat}

\appendix
\section{Proof of the Logistic Loss Bound}
We show the bound $|x-y| \leq \sqrt{\frac{1}{2}(x\ln(\frac{x}{y}+(1-x)\ln(\frac{1-x}{1-y}))}$ as well as that $\frac{1}{2}$ is the optimal constant.
\begin{proof}
To show that $\frac{1}{2}$ is the optimal constant we will leave it for now a variable $c$.
Equivalent to the above statement we can show that $0 \leq {c(x\ln(\frac{x}{y}+(1-x)\ln(\frac{1-x}{1-y}))}-(x-y)^2=:f_y(x)$. We show this for any fixed $y \in (0,1)$. First note that for $x \to 1$ or $x \to 0$ we observe that $f_y(x) \to +\infty$. To prove the above inequality it suffices thus to show that the only minimum of $f_y(x)$ is given by $x=y$. To do so, we analyze the minima of $f_y(x)$. The derivative is given by $\frac{\partial f_y(x)}{\partial x}=c(\ln(\frac{(y-1)x}{(x-1)y})+2(x-y)$. Finding the minima with the first order condition of equating the derivative to zero, we find that the minima are given by the equation
\begin{equation}
\frac{e^{\frac{-2}{c}x}x}{x-1}=\frac{e^{\frac{-2}{c}y}y}{y-1}.
\end{equation}
To show that $x=y$ is the only valid solution we show that $g(x):=\frac{e^{\frac{-2}{c}x}x}{x-1}$ is monotonically strictly increasing in $x$. For that we look at the first derivative of $g(x)$, which is given by
\begin{equation}
\frac{\partial g(x)}{\partial x}=\frac{e^{\frac{-2}{c}x}(c+2(x-1)x)}{c(x-1)^2}.
\end{equation}
The function $g(x)$ is monotonically strictly increasing if the derivative is always positive, which is the case if $c>-2(x-1)x$ for all $x$. As $x=\frac{1}{2}$ maximizes the right hand side we come to the inequality $c>\frac{1}{2}$. Further analysis shows that $c=\frac{1}{2}$ is still a valid choice, as it only leads to $g(x)$ having a saddle point at $x=\frac{1}{2}$, but still being monotonically strictly increasing. 
\end{proof}
\end{document}